\documentclass{article}

\usepackage{arxiv}

\usepackage[utf8]{inputenc} 
\usepackage[T1]{fontenc}    

\usepackage{amsmath}
\usepackage{amssymb}
\usepackage{mathtools}
\usepackage{amsthm}
\usepackage{siunitx}
\usepackage[thinc]{esdiff}
\usepackage{physics}

\usepackage[noend]{algorithmic}
\usepackage{algorithm}

\usepackage{tabularx}
\usepackage{rotating,ragged2e}
\usepackage{caption, makecell}
\usepackage{multirow}

\usepackage{hyperref}       
\usepackage{url}            
\usepackage{booktabs}       
\usepackage{amsfonts}       
\usepackage{nicefrac}       
\usepackage{microtype}      
\usepackage{cleveref}       
\usepackage{lipsum}         
\usepackage{graphicx}
\usepackage{natbib}
\usepackage{doi}

\usepackage{xcolor}

\def\idrm#1{\ensuremath{\mathrm{#1}}}

\def\idcal#1{\ensuremath{\mathcal{#1}}}  

\newtheorem{lemma}{Lemma}

\newcommand{\eat}[1]{}

\newcommand\numberthis{\addtocounter{equation}{1}\tag{\theequation}}

\title{Quantization Range Estimation \\ for Convolutional Neural Networks}

\newif\ifuniqueAffiliation
\uniqueAffiliationtrue

\ifuniqueAffiliation 
\author{ {Bingtao Yang, Yujia Wang, Mengzhi Jiao, and Hongwei Huo}
\thanks{Correspondence should be sent to \url{hwhuo@mail.xidian.edu.cn}} \\
	Department of Computer Science\\
	Xidian University \\
	Xi'an 710071, China \\
}
\fi

\renewcommand{\shorttitle}{Quantization Range Estimation for CNNs}

\hypersetup{
pdftitle={REQuant},
pdfsubject={q-CS.CS, q-CS.AI},
pdfauthor={Bingtao Yang, Yujia Wang, Mengzhi Jiao, and Hongwei Huo},
pdfkeywords={Range estimation, Model compression, Post-training quantization, Convolutional neural networks},
}

\begin{document}
\maketitle

\begin{abstract}
	Post-training quantization for reducing the storage of deep neural network models
  has been demonstrated to be an effective way in various tasks.
  However, \mbox{low-bit} quantization while maintaining model accuracy is a challenging problem.
  In this paper, we present a range estimation method to improve the quantization performance
  for post-training quantization.
  We model the range estimation into an optimization problem of minimizing quantization errors
  by layer-wise local loss. We prove this problem  is locally convex and present an efficient
  search algorithm to find the optimal solution.
  We propose the application of the above search algorithm to the transformed weights space
  to do further improvement in practice.
  Our experiments demonstrate that our method outperforms state-of-the-art
  performance generally on top-1 accuracy for image classification tasks on the ResNet series models
  and Inception-v3 model.
  The experimental results show that the proposed method has almost no loss of top-1 accuracy
  in 8-bit and 6-bit settings for image classifications, and the accuracy of 4-bit quantization
  is also significantly improved.
  The code is available at \url{https://github.com/codeiscommitting/REQuant}.
\end{abstract}


\section{Introduction}
\label{sec:intro}

In recent years, deep neural networks (DNNs) have developed rapidly and achieved striking results
in many related fields such as image classification \cite{Krizhevsky2012,Szegedy2015},
natural language processing \cite{ConneauSB2017,LiuHCG2019,LiuWJC2020}, and
semantic recognition \cite{HintonDYD2012,ZhangPB2017}.
Deep learning methods are typically evaluated based on their accuracy on a given task,
which leads to the gradual development of neural network architectures towards
more complexity and more layers.
This means that these networks have extremely high demands for computing and
storage resources during actual deployment. Therefore, in resource constrained scenarios
such as mobile terminals, Internet of Things devices and edge computing nodes,
the complete deployment of neural networks often encounters feasibility bottlenecks.

Faced with the challenge of how to maintain network performance while reducing
model size and running costs based on existing neural network research results,
researchers have proposed various model compression approaches. These approaches
aim to reduce model complexity while maintaining its performance as much as possible,
enabling the model to run efficiently even in resource constrained conditions.
Network pruning \cite{han2015deep,ThiNet2017,he2022filter,wu2024auto} reduces
the complexity and storage requirements of models by removing redundant neurons and connections,
thereby reducing computational and data transmission costs.
Knowledge distillation \cite{MishraM2018,KimPK2018,AghliR2021,Hernandez2025} utilizes
the knowledge of a large pretrained model to train a smaller model to achieve similar
or identical performance.
Low rank decomposition \cite{TaiXZ2016,YuLW2017,Rajarshi2024,DaiFM2025} decomposes the weight matrix
into low rank approximations, thereby reducing the number of parameters and computational complexity.
Model quantization \cite{JacobKC2018,WangLL2019,DongYA2020,KozlovLS2021,TangOW2022,RokhAK2023,GongLL2025}
aims to approximate 32-bit floating-point parameters with low-bit width representation.

Compared to other model compression approaches, quantization is generally a more promising method
for its high compression and less accuracy reduction \cite{GongLL2025}, and can be applied
to various types of DNNs \cite{RokhAK2023}.
There are two common quantization approaches: Quantization-Aware Training (QAT)
and Post-Training Quantization (PTQ).
QAT integrates quantization operations into the model training process, enabling
the model to adapt to the constraints imposed by quantization during its learning phase.
In contrast, PTQ applies quantization techniques after the model has completed
its training process. This approach leverages the already trained model weights and
seeks to optimize the model's performance under quantization without retraining.
In this paper, we focus on the post-training quantization, since it has become
the standard procedure to produce efficient low-precision neural networks without retraining \cite{GongLL2025}.
Post-training quantization for reducing the storage of deep neural network models
has been demonstrated to be an effective way in various tasks.
However, low-bit quantization while maintaining model accuracy is a challenging problem.

The distribution of full-precision weights and activations plays a critical role
in determining the effectiveness of quantization.
Efficient quantization should maintain the distribution of these values and preserves
informative parts of the original data.
During the quantization process, multiple values are mapped to one value, and determining
quantization levels is the key to minimum the quantization loss thus minimizing
quantization accuracy degradation.
The existing quantization techniques based upon the distribution can be classified into
three general categories: uniform, non-uniform, and adaptive methods.
In uniform quantization \cite{JacobKC2018,BannerNS2019,WangLL2019,DongYA2020,KozlovLS2021,TangOW2022,RokhAK2023,LiGT2021,GongLL2025},
the quantization interval size is constant,
while in non-uniform quantization~\cite{Miyashita2016,ZhouYG2017,LeeMC2017,LiDW2019,WangCH2020,Yvinec2023,CaiHS2017,SQuant2022},
it varies.
Banner et al. \cite{BannerNS2019} theoretically derived the optimal quantization parameters to
optimize the threshold of activation values. Their numerical results also indicate
that for convolutional networks with different quantization thresholds, ``per channel''
and bias correction can improve the accuracy of the quantization model. In addition,
Wang et al. \cite{WangCH2020} used bit segmentation and concatenation techniques to ``segment''
integers into multiple bits, then optimize each bit, and finally stitch all bits back
into integers. Li et al. \cite{GongLL2025} utilized the basic building blocks
in DNN and reconstructed them one by one.
Recently, Edouard et al. \cite{Yvinec2023} proposes a method called PowerQuant,
which achieves non-uniform quantization without training data by searching for power function transformations,
thereby significantly improving quantization accuracy while preserving the computational
structure of neural networks.

{\bf Contributions.}
In this paper, we focus on the post-training quantization and propose an effective method
for quantization range estimation, called REQuant.
\eat{\color {blue}
Based on the above analysis, we propose two strategies to reduce quantization errors,
thereby improving the model's accuracy.
The first strategy is range estimation for quantization which we describe in this section.
We reduce the range of weights that need to be quantized through mapping weights located
at the rim of the distribution inward, to mitigate the quantization loss and improve
the accuracy of quantized neural networks.
We model the mapping process into an optimization problem of minimizing quantization errors
for each layer separately.
The second  strategy is weight transformation, called reshaping, which we describe in the following section.
} 
We summarize below our novel contributions below:

\begin{enumerate}
   \item Based upon our experimental results, we observe that using all original
         weight values in quantization will lead to lower top-1 accuracy on the quantized neural network,
         especially in low bit case.
         Therefore we introduce the concept of range estimation, that is, reducing the range
         of weights that need to be quantized through mapping weights located at the rim of
         the distribution inward, to mitigate the quantization loss and improve the accuracy
         of quantized neural networks.
         We model the mapping process into an optimization problem of minimizing quantization errors
         for each layer separately.
         We prove this problem is locally convex and present an efficient search algorithm
         to find the optimal solution.

   \item We transform the weights to reshape the distribution of weights so that
         the quantization interval can be allocated effectively.
         We derive the convexity for the corresponding optimization problem, so that we
         can apply our proposed search algorithm to the problem of minimizing quantization errors
         in the transformed weights space.

   \item Our experiments demonstrate that our method outperforms state-of-the-art performance
         generally on top-1 accuracy on CIFAR-10 and \mbox{CIFAR-100} classification tasks \cite{Krizhevsky2009}
         on the ResNet series models \cite{HeZR2017} and Inception-v3 model \cite{SzegedyVI2016}.
         The experimental results show that the proposed method has almost no loss of top-1 accuracy
         for image classifications in 8-bit and 6-bit settings, and
         the accuracy of 4-bit quantization is also significantly improved.
         The code is available at \url{https://github.com/codeiscommitting/REQuant}.

\end{enumerate}

\section{Preliminaries}
\label{sec:pre}


Quantization for deep neural network Models is an effective technique. It involves
storing full-precision values in a low bit-width format. This storage effectively
cuts down the memory footprint of the neural network, thus facilitating a significant
speed-up in the execution of multiple tasks like classification and inference.
We start by presenting some preliminaries regarding quantization. Specifically,
we focus on the post-training quantization.
Let $\idcal{W} = \{ W_1, W_2, \ldots, W_L \}$ denote the set of weights of the $L$ convolutional layers
in the neural network.
For each $1 \leq i \leq L$, we let $w_m = \max\{ |w| \ : \  w \in W_i \}$ and $w \in \mathbb{R}$.
We can represent quantization by the mapping of real numbers $w$ to integers $w_q$ in the following way:
\begin{equation*}
w_q = \mathsf{clip}\bigl( \mathsf{round}(w/s) \bigr)  
\end{equation*}
where $\mathsf{round}(\cdot)$ is the round-to-nearest operator, and $s$ is the scale factor
that determines the resolution of quantization and is determined by the bit-width $b$ of
the quantized values and $w_m$. We define $s$ as
\begin{equation*}
\label{eq:scalefactor}
 s  =  \frac{w_m}{2^{b-1} - 1}
\end{equation*}
The $\mathsf{clip}(\cdot)$ is a truncation function.
It maps the values that lie outside the interval $[l, r]$ to the endpoints 
of this interval.
Specifically, it is defined as:
\begin{eqnarray*}
 \mathsf{clip}(x)
& = & \left \{
\begin{array}{ll}
      l,  & \idrm{if} \ x < l  \\
      x,         & \idrm{if} \ l \leq x \leq r \\
      r, & \idrm{if} \ x > r \\
\end{array} \right. \label{eq:clip}
\end{eqnarray*}

For the bit width $b$, we have $l = -2^{b-1}$ and $r = 2^{b-1}-1$.

\section{Methodology}
\label{sec:method}

\subsection{Quantization model and range estimation}
\label{sec:quantmodel}

Quantization based on the maximum value falls short of comprehensively accounting for
the actual data distribution. For example, when a quantization strategy that depends on
the maximum values of weights or activations is implemented in the ResNet-18 network \cite{HeZR2017}
on the CIFAR-10 dataset \cite{Krizhevsky2009} for image classification task , it attains
a top-1 classification accuracy of $95.05\%$, $94.85\%$ and $89.09\%$ for bit-width $b = 8, 6$ and $4$ bits
according to Table \ref{tab:ablation-cifar10} in Section \ref{sec:ablation}.
Compared to the full-precision model with the $95.15\%$ accuracy, the low-bit ResNet-18
with 4-bit quantization has significant accuracy drop.
Majority of the deep neural network weights are densely distributed around zero,
which requires a small $s$ for quantization interval to distinguish their differences.
However, using the maximum weights to compute $s$ will lead to a large quantization interval,
one that is unable to achieve the goal of quantizing densely distributed weights and preserving
their ``difference'' in quantized space, resulting in weights being mapped into
the same quantized value, thus greater quantization error.
Consequently, this simple approach would reduce the discriminative power
of most weights, leading to significant discretization errors that ultimately compromise
the model's accuracy.

Based on the above analysis, we propose two strategies to reduce quantization errors,
thereby improving the model's accuracy.
The first strategy is range estimation for quantization which we describe in this section.
We reduce the range of weights that need to be quantized through mapping weights located
at the rim of the distribution inward, to mitigate the quantization loss and improve
the accuracy of quantized neural networks.
We model the mapping process into an optimization problem of minimizing quantization errors
for each layer separately.
The second  strategy is weight transformation, called reshaping, which we describe in the following section.

Now, we consider the first strategy.
We introduce a factor of parameter $\alpha \in (0, 1]$ in the expression of calculating
the scale factor $s$, which we show in Equation \eqref{eq:scale}.

\begin{equation}
\label{eq:scale}
 s  =  \frac{\alpha w_m}{2^{b-1} - 1}
\end{equation}

Correspondingly, the quantized values of weights are
\begin{equation}
w_q = \mathsf{clip}\bigl( \mathsf{round}\bigl( \frac{w(2^{b-1}-1)}{\alpha w_m} \bigr) \bigr)  \numberthis \label{eq:fquant}
\end{equation}

We model the range estimation into an optimization problem of minimizing quantization errors
with respect to parameter $\alpha$.
We let $f(\alpha, b)$ denote the quantization error function for a layer, measured by the mean squared error
as done in \cite{NagelFA2021}.
We define $f(\alpha, b)$ as
\begin{equation}
\label{eq:fmse}
f(\alpha, b) = \frac{1}{|W|}{\sum_{w \in W} \Bigl( w - w_q \frac{\alpha w_m}{2^{b-1} - 1} \Bigr)^2}
\end{equation}
where $W$ is the set of weights in a layer and $|W|$ is size of set $W$.

The goal is to find the best $\alpha^{*} \in (0, 1]$
that minimizes $f(\alpha, b)$ under some~$b$ such that
\begin{equation} \label{eq:fmin}
\alpha^{*} =  \idrm{arg}\mathop{\min}\limits_{\alpha} f(\alpha, b)
\end{equation}

We seek to design an efficient algorithm to find $\alpha^{*} \in (0,1]$ such that
$f(\alpha^{*}, b)$ approximates the minimum quantization error defined in \eqref{eq:fmse}.

\begin{lemma}
\label{lem:convexf}
The minimization problem defined in \eqref{eq:fmse}
is locally convex around any solution $\alpha^{*}$.  
\end{lemma}
\begin{proof}
The function $f(\alpha, b)$ is differentiable.
According to Equation \eqref{eq:fquant}, we have
\begin{eqnarray*}
w_q
& = & \mathsf{clip}\bigl( \mathsf{round}\bigl( \frac{w(2^{b-1}-1)}{\alpha w_m} \bigr) \bigr)
 \ = \  \left \{
\begin {array}{ll}
    2^{b-1} - 1
        & \mbox{if} \ \alpha w_m < w;   \\
	\mathsf{round}\bigl( \frac{w(2^{b-1}-1)}{\alpha w_m} \bigr)
        & \mbox{if} \ -\alpha w_m \leq w \leq \alpha w_m;  \\
   -2^{b-1}
        & \mbox{if} \ w < -\alpha w_m.
\end{array}
\right.
\end{eqnarray*}
According to the differentiation rules, the rounding operator has a zero derivative almost everywhere \cite{Rudin1987},
so we know that $\frac{\partial w_q}{\partial \alpha} = 0$.
We compute the first derivative of $f(\alpha, b)$, then
$$
\diffp{f(\alpha, b)}{\alpha}
 = \frac{1}{|W|} \sum_{w \in W} 2 \Bigl( w - w_q \frac{\alpha w_m}{2^{b-1} - 1} \Bigr)
   \Bigl( - w_q \frac{w_m}{2^{b-1} - 1} \Bigr)
$$
Now, we compute the second derivative of $f(\alpha, b)$, then
$$
\pdv[2]{f(\alpha, b)}{\alpha}
 = \frac{1}{|W|} \sum_{w \in W} 2 \Bigl( - w_q \frac{w_m}{2^{b-1} - 1} \Bigr)^{2}
$$
As $\alpha \in (0, 1]$, and $\forall w \in W$ there must exist a $w \in W$ such that $w_q \neq 0$,
therefore, $\pdv[2]{f(\alpha, b)}{\alpha} > 0$.
\end{proof}

By Lemma~\ref{lem:convexf}, the second derivative of the quantization error function $f$
is greater than 0, it indicates that $f$ has a local minimum over $\alpha \in (0,1]$.
Based upon this, we give an algorithm that finds an optimal $\alpha$ that minimizes $f(\alpha)$.
The key observation is:
If $f(x) < f(y)$ for some $x,y \in (0,1]$, then $x < y$ implies $\alpha^{*} < y$, and
$y \leq x$ implies $\alpha^{*} \geq x$.
This suggests a natural strategy: maintain a candidate interval $[c, d]$ such that $\alpha^{*} \in [c,d]$
and iteratively narrow it down until finding the desired $\alpha^{*}$.
Then $d_{k+1} - c_{k+1} = \phi (d_{k} - c_{k})$.
After the $k$ iterations,
$d_{k+1} - c_{k+1} = \phi (d_{k} - c_{k}) < \phi^{k}$, where $0 < \phi < 1$ is a constant factor.
Among several search methods \cite{Wang2021}, it can be seen in Table \ref{tab:ablation-search},
the golden section search has the minimum quantization loss with the fastest search time.
Thus we use the golden section search algorithm \cite{Wang2021} to approximate the optima $\alpha^{*}$,
which we show in Algorithm~\ref{alg:gss}.

\begin{algorithm}[htbp]
   \caption{$\mathsf{GSSEARCH}(W, b,\epsilon, \phi )$ 
   }  \label{alg:gss}
\begin{minipage}{0.56\textwidth}
\begin{algorithmic}[1]
   \STATE $c \gets 0, \ d \gets 1 $, \ $\phi \gets 0.618 $
   \STATE $x_1 \gets d - \phi (d - c), \ x_2 \gets c + \phi (d - c) $
   \STATE $f_1 \gets f(x_1)$, \ $f_2 \gets f(x_2)$
   \WHILE{$|d - c| > \epsilon$}
       \IF{$f_1 < f_2$}
           \STATE $d \gets x_2$                      \hspace{1.7cm} $\triangleright$ update right end $d$
           \STATE $x_2 \gets x_1$, \ $f_2 \gets f_1$ \hspace{2pt}   $\triangleright$ keep smaller error
           \STATE $x_1 \gets d - \phi(d - c), \ f_1 \gets f(x_1)$
       \ELSE
           \STATE $c \gets x_1$                       \hspace{1.84cm}$\triangleright$ update left end $c$
           \STATE $x_1 \gets x_2$, \ $f_1 \gets f_2$  \hspace{2pt}   $\triangleright$ keep smaller error
           \STATE $x_2 \gets c + \phi(d - c), \ f_2 \gets f(x_2)$
       \ENDIF
   \ENDWHILE
   \IF{$f_1 \leq f_2$}
       \STATE {\bfseries return} $x_1$
   \ELSE
       \STATE {\bfseries return} $x_2$
   \ENDIF
\end{algorithmic}
\end{minipage}
\hfill
\begin{minipage}{0.40\textwidth}
\textbf{function} $f(x)$
\begin{algorithmic}[1]
   \STATE $e_q \gets 0$
   \STATE $s \gets (x\, w_m) / (2^{b-1} -1)$
   \FOR{$\textbf{all}  \ w \in W$}
       \STATE $w_q \gets \mathsf{clip}(\mathsf{round}(w/s)) $
       \STATE $e_q \gets e_q + (w - w_q \times s)^2$
   \ENDFOR
   \RETURN $e_q / |W|$
\end{algorithmic}
\end{minipage}

\end{algorithm}

\subsection{Range estimation optimization by reshaping distribution}
\label{sec:reshape}

In this section, we further optimize our quantization model by introducing
reshaping distribution of weights.
One way to change the distribution of weights is to apply a function to the weights.
When it comes to redistributing weights, it has been shown in the literature \cite{Yvinec2023}
that power functions are better choice than logarithmic functions.
The former takes into account the bell-shaped distribution of the weights of
the convolutional neural network, so the data is transformed non-linearly.
We can apply the quantization method described in Section \ref{sec:quantmodel}
to the transformed weight space to realize the non-uniform quantization in the original weight space.
In this way, we can adaptive allocate the quantization intervals so that smaller
quantization interval is applied when quantizing the densely distributed weights,
and larger quantization interval is applied when quantizing the sparsely distributed weights
so as to improve the quantization accuracy of the quantized network model.

By applying the square root function to the absolute value of the weights and applying
the resulting weights to our quantization model, we get the quantized values as

\begin{equation}
w_q = \mathsf{clip}\bigl( \mathsf{round}\bigl(\mathsf{sign}(w) \frac{\sqrt{|w|} (2^{b-1}-1)}{\sqrt{\alpha w_m}} \bigr) \bigr)  \numberthis \label{eq:gquant}
\end{equation}

\begin{equation}
w' = \mathsf{sign}(w_q) \bigl(|w_q| \frac{\sqrt{\alpha w_m}}{2^{b-1} - 1} \bigr)^{2} \numberthis \label{round2}
\end{equation}

In a similar way as done in Section \ref{sec:quantmodel}, we define
the quantization error function $g(\alpha, b)$ for the transformed weights in a layer as

\begin{equation}
\label{eq:gmse}
g(\alpha, b) = \frac{1}{|W|}{\sum_{w \in W} \Bigl( w - \mathsf{sign}(w_q) \bigl(|w_q| \frac{\sqrt{\alpha w_m} }{2^{b-1} -1} \bigr)^{2} \Bigr)^2}
\end{equation}

The goal is to find the best $\alpha^{*} \in (0, 1]$
that minimizes $g(\alpha, b)$ under some~$b$ such that
\begin{equation} \label{eq:gmin}
\alpha^{*} =  \idrm{arg}\mathop{\min}\limits_{\alpha} g(\alpha, b)
\end{equation}

We seek to design an efficient algorithm to find $\alpha^{*} \in (0,1]$ such that
$g(\alpha^{*}, b)$ approximates the minimum quantization error defined in \eqref{eq:gmse}.

\begin{lemma}
\label{lem:convexg}
The minimization problem defined in \eqref{eq:gmse}
is locally convex around any solution $\alpha^{*}$.  
\end{lemma}
\begin{proof}
The function $g(\alpha, b)$ is differentiable.
First, simplifying Equation \eqref{eq:gmse}, we have
$$
g(\alpha, b) = \frac{1}{|W|}{\sum_{w \in W} \Bigl( w - \mathsf{sign}(w_q) \frac{(w_q)^{2}\alpha w_m}{(2^{b-1} -1)^{2}} \Bigr)^2}
$$
where by \eqref{eq:gmse}
\begin{eqnarray*}
w_q
& = & \mathsf{clip}\bigl( \mathsf{round}\bigl(\mathsf{sign}(w) \frac{\sqrt{|w|} (2^{b-1}-1)}{\sqrt{\alpha w_m}} \bigr) \bigr) \\
& = &  \left \{
\begin {array}{ll}
    2^{b-1} - 1
        & \mbox{if} \ \alpha w_m < w;   \\
	\mathsf{round}\bigl(\mathsf{sign}(w) \frac{\sqrt{|w|} (2^{b-1}-1)}{\sqrt{\alpha w_m}} \bigr)
        & \mbox{if} \ -\alpha w_m \leq w \leq \alpha w_m;  \\
   -2^{b-1}
        & \mbox{if} \ w < -\alpha w_m.
\end{array}
\right.
\end{eqnarray*}
According to the differentiation rules, the rounding operator has a zero derivative almost everywhere,
so we know that $\frac{\partial w_q}{\partial \alpha} = 0$.
We compute the first derivative of $g(\alpha, b)$,
$$
\diffp{g(\alpha, b)}{\alpha}
 = \frac{1}{|W|} \sum_{w \in W} 2 \Bigl( w - \mathsf{sign}(w_q) \frac{(w_q)^{2}\alpha w_m}{(2^{b-1} -1)^{2}}\Bigr)
   \Bigl( - \mathsf{sign}(w_q) \frac{(w_q)^{2} w_m}{(2^{b-1} -1)^{2}} \Bigr)
$$
Now, we compute the second derivative of $g(\alpha, b)$.
$$
\pdv[2]{g(\alpha, b)}{\alpha}
 = \frac{1}{|W|} \sum_{w \in W} 2 \Bigl(\frac{ (w_q)^{2} w_m}{(2^{b-1} - 1)^{2}} \Bigr)^{2}
$$
As $\alpha \in (0, 1]$, and $\forall w \in W$ there must exist a $w \in W$ such that $w_q \neq 0$,
therefore, $\pdv[2]{g(\alpha, b)}{\alpha} > 0$.
\end{proof}

Therefore, we can still use Algorithm \ref{alg:gss} to find and approximate the optimal solution for
the minimization problem defined in \eqref{eq:gmin}.

\section{Experiments}
\label{sec:experiments}

\subsection{Data sets, pretraining, and practical optimization}
\label{sec:datasets}

We verified the proposed quantization method on CIFAR-10 and CIFAR-100 datasets \cite{Krizhevsky2009},
which are two widely used image classification benchmark datasets for training and evaluating
machine learning and deep learning models.
In our experiment, we first trained the full precision ResNet series networks \cite{HeZR2017}
and Inception-v3 network \cite{SzegedyVI2016} as the benchmark model, and then obtained
their corresponding standard baseline accuracy through testing.
We trains each model 200 epochs with an initial learning rate of $0.1$, and then attenuates
a factor of $0.2$ at the 60th, 120th and 160th epochs respectively.
We select the SGD optimizer and set batch size = 128.
We adopt batch normalization folding \cite{Krishnamoorthi2018,JacobKC2018Quantization}
after each convolution before activation in the weights quantization process
which takes charge of normalizing the input of each layer to make the training process
faster and more stable, and thus possibly improve the quantized accuracy.
In addition, we quantize the weights and activations of all layers except that
the data for input layer and the last layer kept to 8-bit, as done in \cite{HubaraCS2016,ZhuangSL2018}, respectively.

\begin{table}[htbp]
\centering
\small
\caption{ Top-$1$ accuracy ($\%$) for post-training quantization on \mbox{CIFAR-10}. } \label{tab:cifar10}

\begin{tabular}{cccccccc}
 \toprule \rule{0em}{1em}
 model & FP32 ($\%$)
       & Bits (W/A) & Histogram & AdaRound & OMSE & SQuant & REQuant   \\ \midrule \rule{0em}{1em}
\multirow{3}{*}{ResNet-18} & \multirow{3}{*}{95.15}
       & 8/8 & 95.11  & 95.14  & 95.14  & 93.98  & 95.15  \\ \rule{0em}{1em}
    &  & 6/6 & 94.86  & 94.61  & 94.82  & 94.06  & 95.10  \\ \rule{0em}{1em}
    &  & 4/4 & 90.02  & 85.01  & 92.62  & 93.70  & 94.67  \\ \midrule
\multirow{3}{*}{ResNet-34} & \multirow{3}{*}{95.41}
       & 8/8 & 95.35  & 95.27  & 95.32  & 94.28  & 95.37  \\ \rule{0em}{1em}
    &  & 6/6 & 95.27  & 95.03  & 95.12  & 94.22  & 95.38  \\ \rule{0em}{1em}
    &  & 4/4 & 92.02  & 87.71  & 93.71  & 94.09  & 94.96  \\ \midrule
\multirow{3}{*}{ResNet-50} & \multirow{3}{*}{95.24}
       & 8/8 & 95.15  & 95.13	& 95.18	 & 93.38  & 95.21  \\ \rule{0em}{1em}
    &  & 6/6 & 94.80  & 94.68	& 94.89  & 93.33  & 95.21  \\ \rule{0em}{1em}
    &  & 4/4 & 84.99  & 68.44	& 92.63	 & 92.96  & 94.47  \\ \midrule
\multirow{3}{*}{ResNet-101} & \multirow{3}{*}{95.47}
       & 8/8 & 95.46  & 95.43	& 95.37	 & 93.85  & 95.41  \\ \rule{0em}{1em}
    &  & 6/6 & 95.11  & 94.99	& 95.21	 & 93.83  & 95.39  \\ \rule{0em}{1em}
    &  & 4/4 & 80.46  & 67.07	& 92.28	 & 93.44  & 94.69  \\ \midrule
\multirow{3}{*}{Inception-v3} & \multirow{3}{*}{95.58}
       & 8/8 & 95.58  & 95.51	& 95.52	 & 94.81  & 95.57  \\ \rule{0em}{1em}
    &  & 6/6 & 95.10  & 95.11	& 95.29  & 94.42  & 95.10  \\ \rule{0em}{1em}
    &  & 4/4 & 82.84  & 68.37	& 92.21	 & 93.29  & 89.40  \\ \bottomrule
\end{tabular}
\end{table}

\subsection{Comparison with {\bf other PTQ methods}}
\label{sec:comparison}

To validate the effectiveness of our method, we compare our approach under
weight and activation quantization settings.
The experiments cover modern deep learning architectures, including ResNet family \cite{HeZR2017}
and Inceptionv3 \cite{SzegedyVI2016}.
We compare with baselines including OMSE \cite{OMSE2019}, AdaRound \cite{AdaRound2020},
Histogram \cite{PyTorch2023}, and SQuant \cite{SQuant2022} in which most of them have good performances
in low-bit quantization.

In order to better measure the quantitation effect, we have conducted 8-bit, 6-bit
and 4-bit quantitation tests respectively. In addition, we also show the performance
of several other PTQ quantization methods on the same pre training model.
As shown in Table \ref{tab:cifar10}, by observing the experimental results on the CIFAR-10 dataset,
it can be seen that for the ResNet model, the quantization method in this paper
has almost no precision loss in 8-bit and 6-bit quantization compared to the full-precision model (FP32).
At the same time, the accuracy of 4-bit quantization decreases by no more than $1\%$.
For Inception-v3 model, the precision loss of 6-bit quantization is only $0.48\%$, and
that of 4-bit quantization is only $6.18\%$.

Table \ref{tab:cifar100} shows the experimental results on CIFAR-100. Although the overall performance
of the model on the CIFAR-100 dataset is not as good as that on the CIFAR-10 dataset,
it can be found that the method in this paper always maintains the highest accuracy
when 6-bit quantizing the ResNet model. At the same time, during 4-bit quantization,
the accuracy loss of ResNet-18 model is only $1.19\%$, ResNet-34 model is only $1.62\%$,
ResNet-50 model is about $2\%$, and ResNet-101 model is only about $3\%$.
In addition, Inception-v3 model realizes 6-bit quantization with almost no loss.

\begin{table}[htbp]
\centering
\small
\caption{ Top-$1$ accuracy ($\%$) on post-training quantization on \mbox{CIFAR-100}. } \label{tab:cifar100}

\begin{tabular}{cccccccc}
 \toprule \rule{0em}{1em}
 model & FP32 ($\%$)
       & Bits (W/A) & Histogram & AdaRound & OMSE & SQuant & REQuant   \\ \midrule \rule{0em}{1em}
\multirow{3}{*}{ResNet-18} & \multirow{3}{*}{76.08}
       & 8/8 & 76.13	& 76.12	& 76.16	& 76.07	& 76.08  \\ \rule{0em}{1em}
    &  & 6/6 & 75.65	& 75.58	& 75.87	& 75.92	& 75.94  \\ \rule{0em}{1em}
    &  & 4/4 & 49.07	& 40.31	& 67.68	& 74.97 & 74.89  \\  \midrule
\multirow{3}{*}{ResNet-34} & \multirow{3}{*}{77.58}
       & 8/8 & 77.68	& 77.53	& 77.50 & 77.49 & 77.52  \\ \rule{0em}{1em}
    &  & 6/6 & 76.95	& 76.85	& 76.85	& 77.39	& 77.42  \\ \rule{0em}{1em}
    &  & 4/4 & 61.88	& 45.96 & 68.62	& 76.29	& 75.96  \\  \midrule
\multirow{3}{*}{ResNet-50} & \multirow{3}{*}{78.98}
       & 8/8 & 78.89	& 78.65	& 78.83	& 78.80 & 78.87  \\ \rule{0em}{1em}
    &  & 6/6 & 78.45	& 77.54	& 78.37	& 78.57	& 78.96  \\ \rule{0em}{1em}
    &  & 4/4 & 46.73	& 26.88	& 67.93	& 77.22	& 76.84  \\  \midrule
\multirow{3}{*}{ResNet-101} & \multirow{3}{*}{79.01}
       & 8/8 & 78.98	& 79.02	& 78.94	& 78.90 & 78.96  \\ \rule{0em}{1em}
    &  & 6/6 & 78.28	& 78.01 & 78.26	& 78.81	& 78.84  \\ \rule{0em}{1em}
    &  & 4/4 & 52.07	& 26.58	& 69.99	& 76.81	& 75.94  \\  \midrule
\multirow{3}{*}{Inception-v3} & \multirow{3}{*}{80.07}
       & 8/8 & 80.03	& 80.0  & 79.92	& 79.83	& 80.01  \\ \rule{0em}{1em}
    &  & 6/6 & 78.99	& 78.45 & 78.89	& 78.75	& 79.55  \\ \rule{0em}{1em}
    &  & 4/4 & 18.79	& 3.69	& 66.54 & 73.44	& 50.72  \\ \bottomrule
\end{tabular}
\end{table}

\begin{table}[htbp]
\centering
\small
\caption{The top-1 accuracy ($\%$) of four combinational strategies for quantization
on \mbox{CIFAR-10}. } \label{tab:ablation-cifar10}

\begin{tabular}{ccccccc}
 \toprule \rule{0em}{1em}
model  &   FP32 ($\%$)  & Bits (W/A) & no clip + no reshape & clip + no reshape & no clip + reshape & REQuant
 \\ \midrule \rule{0em}{1em}
\multirow{3}{*}{ResNet-18}  & \multirow{3}{*}{95.15 }
       & 8/8 & 95.05 & 95.03 & 95.11 & 95.15  \\ \rule{0em}{1em}
    &  & 6/6 & 94.85 & 94.81 & 95.08 & 95.10  \\ \rule{0em}{1em}
    &  & 4/4 & 89.09 & 92.48 & 93.33 & 94.67  \\ \midrule
\multirow{3}{*}{ResNet-34}  & \multirow{3}{*}{95.41 }
       & 8/8 & 95.33 & 95.30 & 95.39 & 95.37  \\ \rule{0em}{1em}
    &  & 6/6 & 95.22 & 95.27 & 95.25 & 95.38  \\ \rule{0em}{1em}
    &  & 4/4 & 90.26 & 93.45 & 94.18 & 94.96  \\ \midrule
\multirow{3}{*}{ResNet-50}  & \multirow{3}{*}{95.24 }
       & 8/8 & 95.13 & 95.12 & 95.29 & 95.21  \\ \rule{0em}{1em}
    &  & 6/6 & 94.75 & 94.98 & 95.03 & 95.21  \\ \rule{0em}{1em}
    &  & 4/4 & 78.70 & 90.86 & 92.79 & 94.47  \\ \midrule
\multirow{3}{*}{ResNet-101}  & \multirow{3}{*}{95.47 }
       & 8/8 & 95.44 & 95.37 & 95.38 & 95.41  \\ \rule{0em}{1em}
    &  & 6/6 & 95.02 & 95.19 & 95.27 & 95.39  \\ \rule{0em}{1em}
    &  & 4/4 & 73.0  & 89.04 & 92.05 & 94.69  \\ \hline \midrule
\multirow{3}{*}{Inception-v3}& \multirow{3}{*}{95.58 }
       & 8/8 & 95.61 & 95.38 & 95.24 & 95.57  \\ \rule{0em}{1em}
    &  & 6/6 & 94.95 & 94.81 & 95.50 & 95.10  \\ \rule{0em}{1em}
    &  & 4/4 & 66.63 & 87.63 & 90.07 & 89.40  \\ \bottomrule
\end{tabular}
\vspace{-12pt}
\end{table}

\subsection{Ablation study}
\label{sec:ablation}

{\bf Effect of four combinational strategies for quantization on the top-1 classification accuracy.}

To comprehensively evaluate the effectiveness of our proposed quantization strategy,
we conduct an ablation study by comparing four different combinational strategies
for quantization on the CIFAR-10 and CIFAR-100 datasets.
These combinations include:
(1) ``no clip + no reshape'', representing the baseline quantization method described in Section \ref{sec:pre};
(2) ``clip plus no reshape'', which incorporates the clipping operation introduced in Section \ref{sec:quantmodel};
(3) ``no clip + reshape'', where the shapping parameter $\alpha$ is setting to $1$ described in Section \ref{sec:reshape}; and
(4) REQuant, our quantization method, combining both clipping and reshaping described in Section \ref{sec:reshape}.
We show the top-1 classification accuracy of the four quantization strategies
on ResNet-18 on CIFAR-10 in Table \ref{tab:ablation-cifar10}.

As can be seen in Table \ref{tab:ablation-cifar10}, comparing different strategies
across the same model and bit-width setting, we observe that each added component
contributes to the top-1 accuracy improvements. The clipping operation alone alleviates
the effect of outliers, while the reshaping operation (even without clipping) enhances
the scaling flexibility. Meanwhile, although accuracy generally decreases
with lower bit-width, the clip plus reshape strategy effectively mitigates this degradation.
For most tested models, the accuracy at 8/8 and 6/6 bit-width with clip + reshape
is nearly on par with the full-precision baseline, and even at 4/4 precision,
the performance remains competitive. For example, ResNet50 at 4/4 precision improves dramatically
from $78.70\%$ (no clip + no reshape) to $94.47\%$ with our method.
Similarly, Inception-v3 shows a substantial gain from $66.63\%$ to $89.40\%$.

\begin{table}[htbp]
\centering
\small
\caption{The top-1 accuracy ($\%$) of four combinational strategies for quantization
on \mbox{CIFAR-100}.
} \label{tab:ablation-cifar100}

\begin{tabular}{ccccccc}
 \toprule \rule{0em}{1em}
model  &   FP32 ($\%$)  & Bits (W/A) & no clip + no reshape & clip+no reshape & no clip+reshape & REQuant
 \\ \midrule \rule{0em}{1em}
\multirow{3}{*}{ResNet-18}  & \multirow{3}{*}{76.08}
       & 8/8 & 75.96 & 76.10 & 76.13 & 76.08  \\ \rule{0em}{1em}
    &  & 6/6 & 75.71 & 75.70 & 76.11 & 75.94  \\ \rule{0em}{1em}
    &  & 4/4 & 51.55 & 69.17 & 71.83 & 74.89  \\ \midrule
\multirow{3}{*}{ResNet-34}  & \multirow{3}{*}{77.58}
       & 8/8 & 77.56 & 77.44 & 77.39 & 77.52  \\ \rule{0em}{1em}
    &  & 6/6 & 76.75 & 76.83 & 77.53 & 77.42  \\ \rule{0em}{1em}
    &  & 4/4 & 58.73 & 68.07 & 73.46 & 75.96  \\ \midrule
\multirow{3}{*}{ResNet-50}  & \multirow{3}{*}{78.98}
       & 8/8 & 78.88 & 78.90 & 78.90 & 78.87  \\ \rule{0em}{1em}
    &  & 6/6 & 77.84 & 78.55 & 78.70 & 78.96  \\ \rule{0em}{1em}
    &  & 4/4 & 38.85 & 69.29 & 71.43 & 76.84  \\ \midrule
\multirow{3}{*}{ResNet-101}  & \multirow{3}{*}{79.01}
       & 8/8 & 79.02 & 78.90 & 78.84 & 78.96  \\ \rule{0em}{1em}
    &  & 6/6 & 78.42 & 78.59 & 78.90 & 78.84  \\ \rule{0em}{1em}
    &  & 4/4 & 40.04 & 65.91 & 70.17 & 75.94  \\ \hline \midrule
\multirow{3}{*}{Inception-v3}& \multirow{3}{*}{80.07}
       & 8/8 & 79.98 & 79.91 & 79.99 & 80.01  \\ \rule{0em}{1em}
    &  & 6/6 & 78.30 & 79.09 & 79.61 & 79.55  \\ \rule{0em}{1em}
    &  & 4/4 & 6.76  & 31.31 & 51.54 & 50.72  \\ \bottomrule
\end{tabular}
\end{table}

Table \ref{tab:ablation-cifar100} shows the top-1 classification accuracy of
the four quantization strategies on ResNet-18 on \mbox{CIFAR-100}.
By Table \ref{tab:ablation-cifar100}, we can see it presents similar observations
on the more challenging \mbox{CIFAR-100} dataset to those on \mbox{CIFAR-10}.
Here, the benefits of our strategy are even more significant, especially
under low-bit scenarios. For instance, ResNet50 and ResNet101 show performance
improvements of over $35\%$ and $36\%$, respectively, when transition from the baseline
to the full method at 4/4 precision.
Inception-v3, which suffers severely from quantization in the absence of clipping and reshaping,
recovers up to $50.72\%$ accuracy with REQuant.

In summary, the ablation results clearly demonstrate that both clipping and reshaping contribute significantly to preserving model accuracy under quantization, and their combination is particularly effective in low-bit scenarios. Our clip plus reshape strategy consistently delivers the best performance across various models and datasets, confirming its general applicability.

{\bf Quantization loss and search time of three one-dimensional search algorithms on ResNet-18.}
To see which one-dimensional search algorithm is chosen to solve the optimization
problem defined in \eqref{eq:fmin}, we conducted some experiments using three
commonly used methods: bisection linear search, Golden section search, and Nelder-Mead search \cite{Wang2021}
to search the optima. We show the quantization loss $f(\alpha, b)$ defined in \eqref{eq:fmse}
and search time on five convolutional layers of the ResNet-18 model of the three search algorithms
on CIFAR-10 in Table \ref{tab:ablation-search},
where ``Bisection'', ``GoldenSearch'' and ``Nelder-Mead'' denote bisection linear search,
Golden section search, and Nelder-Mead search \cite{Wang2021}, respectively.
It can be seen in Table \ref{tab:ablation-search}, the golden section search has the minimum quantization
loss with the fastest search time.

\begin{table}[htbp]
\centering
\small
\caption{Quantization loss and search time in milliseconds (ms) on five convolutional layers
of the ResNet-18 model for three one-dimensional search algorithms on CIFAR-10. } \label{tab:ablation-search}

\begin{tabular}{ccccc}
 \toprule \rule{0em}{1em}
 layer  &   method  & $\alpha$ & $f(\alpha,8)$ & time (ms)
 \\ \midrule \rule{0em}{1em}
\multirow{3}{*}{layer1.0.conv1}
    & Bisection    & 0.93751251 & 2.160607e-07 & 76.79  \\ \rule{0em}{1em}
    & GoldenSearch & 0.93638047 & 2.159448e-07 & 39.89  \\ \rule{0em}{1em}
    & Nelder-Mead  & 0.93636170 & 2.159448e-07 & 55.44  \\  \midrule
\multirow{3}{*}{layer1.0.conv2}
    & Bisection    & 0.93698534 & 1.957043e-07 & 92.77  \\ \rule{0em}{1em}
    & GoldenSearch & 0.94232728 & 1.954397e-07 & 46.68  \\ \rule{0em}{1em}
    & Nelder-Mead  & 0.94234924 & 1.954396e-07 & 66.73  \\ \midrule
\multirow{3}{*}{layer1.1.conv1}
    & Bisection    & 0.91983789 & 1.588616e-07 & 50.49  \\ \rule{0em}{1em}
    & GoldenSearch & 0.91988005 & 1.588617e-07 & 43.88  \\ \rule{0em}{1em}
    & Nelder-Mead  & 0.91979981 & 1.588619e-07 & 58.71  \\ \midrule
\multirow{3}{*}{layer1.1.conv2}
    & Bisection    & 0.84358125 & 1.372036e-07 & 75.05  \\ \rule{0em}{1em}
    & GoldenSearch & 0.84569131 & 1.371056e-07 & 40.89  \\ \rule{0em}{1em}
    & Nelder-Mead  & 0.84575195 & 1.371057e-07 & 124.97  \\ \midrule
\multirow{3}{*}{layer2.0.conv1}
    & Bisection    & 0.75373342 & 1.032244e-07 & 67.78  \\ \rule{0em}{1em}
    & GoldenSearch & 0.75373421 & 1.032244e-07 & 42.23  \\ \rule{0em}{1em}
    & Nelder-Mead  & 0.75381195 & 1.040503e-07 & 69.71  \\  \bottomrule
\end{tabular}
\vspace{-12pt}
\end{table}

\section{Conclusions}
\label{sec:conclusions}

In this paper, we propose an effective method for quantization range estimation.
We introduce the concept of range estimation and model the range estimation into
an optimization problem of minimizing quantization errors.
We prove this problem is locally convex and present an efficient search algorithm
to find the optimal solution.
We transform the weights to reshape the distribution of weights so that the quantization interval
can be allocated effectively for the densely and sparsely distributed weights
to do further improvements in practice.
We derive the convexity for the corresponding optimization problem in the transformed weights space,
so that we can apply our proposed search algorithm to the minimum quantization errors.
Our experiments demonstrate that our method achieves state-of-the-art performance on top-1 accuracy
for image classification tasks on the ResNet series models and Inception-v3 model.
Some interesting work are to further improve quantization model under low-bit setting
and verify it on the ImageNet dataset \cite{Deng2009ImageNet}.

\bibliographystyle{unsrtnat}
\bibliography{references}  






\end{document}